%% file: paper.tex
\documentclass{article}
\usepackage{ijcai19}

\usepackage{times}
\usepackage{soul}
\usepackage{url}
\usepackage[utf8]{inputenc}
\usepackage[small]{caption}
\usepackage{graphicx}
\usepackage{amsmath,amssymb,amsthm}
\usepackage{booktabs}
\usepackage{algorithm}
\usepackage{algorithmic}
\usepackage{bm}
\usepackage{multirow}
\usepackage{enumitem}
\usepackage{color}
\usepackage{subfigure}
\usepackage{nicefrac}

\graphicspath{{./figures/}}
\newcommand{\figwidth}{0.49}

\newcommand{\vect}[1]{{\bm{\mathbf{#1}}}}

\newcommand{\vb}{\vect{b}}

\newcommand{\vx}{\vect{x}}

\newcommand{\vw}{\vect{w}}
\newcommand{\vu}{\vect{u}}
\newcommand{\vv}{\vect{v}}

\newcommand{\mD}{\mathcal{D}}
\newcommand{\mO}{\mathcal{O}}
\newcommand{\mF}{\mathcal{F}}
\newcommand{\mM}{\mathcal{M}}

\newcommand{\edp}{$\epsilon$-differentially privacy}
\newcommand{\pr}{\text{\normalfont Pr}}
\newcommand{\dett}{\text{det}}

\newtheorem{theorem}{Theorem}
\newtheorem{lemma}[theorem]{Lemma}
\newtheorem{proposition}[theorem]{Proposition}
\newtheorem{definition}{Definition}
\newtheorem{remark}{Remark}
\newtheorem{corollary}[theorem]{Corollary}

\usepackage{array}
\newcolumntype{L}[1]{>{\raggedright\let\newline\\\arraybackslash\hspace{0pt}}m{#1}}
\newcolumntype{C}[1]{>{\centering\let\newline  \\\arraybackslash\hspace{0pt}}m{#1}}
\newcolumntype{R}[1]{>{\raggedleft\let\newline \\\arraybackslash\hspace{0pt}}m{#1}}

\title{Privacy-preserving Stacking with Application to \\
	Cross-organizational Diabetes Prediction}

\author{
Quanming Yao$^{1,2}$
\and
Xiawei Guo$^{1,*}$
\and
James T. Kwok$^2$
\and
Wei-Wei Tu$^1$
\and \\
Yuqiang Chen$^1$
\and
Wenyuan Dai$^1$
\And
Qiang Yang$^2$
\affiliations
$^1$4Paradigm Inc\\
$^2$Department of Computer Science and Engineering,
HKUST
}

\begin{document}

\maketitle

\begin{abstract}
To meet the standard of differential privacy,
noise is usually added into the original data,
which inevitably deteriorates the predicting performance
of subsequent learning algorithms.
In this paper,
motivated by the success of improving
predicting performance by ensemble learning,
we propose to enhance privacy-preserving logistic regression by
stacking.
We show
that this can be done
either by
sample-based
or feature-based
partitioning.
However,
we prove
that when privacy-budgets are the same,
feature-based partitioning requires fewer  samples
than sample-based one,
and thus likely has better empirical performance.
As transfer learning is difficult to be integrated with a differential privacy guarantee,
we further combine the proposed method with hypothesis transfer learning 
to address the problem of learning across different organizations.
Finally, we not only demonstrate the effectiveness of our method on two benchmark data sets,
i.e., MNIST and NEWS20,
but also apply it into a real application of cross-organizational diabetes prediction
from RUIJIN data set,
where privacy is of a significant concern.
\footnote{Correspondace to X. Guo at guoxiawei@4paradigm.com}
\end{abstract}

\vspace{-5px}
\section{Introduction}

In recent years, data privacy
has become a serious concern
in both academia and industry
\cite{Dwork2006,Chaudhuri2011,dwork2014algorithmic,Abadi2016}.
There are now privacy laws, such as Europe's
General Data Protection Regulation (GDPR),
which regulates the protection of private data and restricts data transmission
between organizations.
These raise challenges for
cross-organizational machine learning \cite{pathak2010multiparty,hamm2016learning,Papernot2017,Xie2017},
in which data have to be distributed to different organizations,
and
the learning model
needs to
make
predictions
in private.

A number of approaches have been proposed to ensure privacy protection.
In machine learning,
differential privacy \cite{dwork2014algorithmic} is often used
to allow data be exchanged among organizations.
To design a differentially private algorithm,
carefully designed noise is usually added to the original data to disambiguate
the algorithms.
Many standard learning algorithms have been extended for
differential privacy.
These
include logistic regression \cite{Chaudhuri2011},
trees \cite{emekcci2007privacy,fong2012privacy},
and deep networks \cite{shokri2015privacy,Abadi2016}.
In particular, linear models
are simple and easy to understand,
and their differentially private variants
(such as privacy-preserving logistic regression  (PLR))
\cite{Chaudhuri2011})
have rigorous theoretical guarantees
\cite{Chaudhuri2011,bassily2014private,hamm2016learning,kasiviswanathan2016efficient}.
However,
the injection of noise often degrades prediction performance.

Ensemble learning
can often signficantly improve
the performance of a single learning model
\cite{zhou2012ensemble}.
Popular examples
include bagging
\cite{breiman1996bagging}, boosting
\cite{friedman2000additive}, and stacking
\cite{wolpert1992stacked}.
These motivate us
to develop an
ensemble-based
method which
can benefit from data protection,
while enjoying good prediction performance.
Bagging
and boosting
are based on partitioning of training samples,
and use pre-defined rules
(majority or weighted voting) to combine predictions from models
trained on different partitions.
Bagging improves learning performance by reducing the variance.
Boosting, on the other hand, is useful in converting weak models to a strong one.
However, the logistic regression model, which is the focus in this paper, often has good performance in many applications,
and is a relatively strong classifier.
Besides,
it is a convex model and relatively stable.

Thus,
in this paper,
we focus on stacking.
While stacking also partitions the training data, this can be
based on either samples \cite{breiman1996stacked,smyth1999linearly,ozay2012new} or features \cite{boyd2011distributed}.
Multiple low-level models
are then learned
on the different data partitions,
and a high-level model
(typically, a logistic regression model)
is used to combine their predictions.
By combining with PLR,
we show
how differential privacy can be ensured
in stacking.
Besides,
when
the importance
of features
is known a priori,
they can be easily incorporated in feature-based partitioning.
We further analyze
the learning guarantee of
sample-based and feature-based stacking,
and show theoretically
that feature-based partitioning
can have
lower sample complexity (than sample-based partitioning),
and thus better performance.
By adapting the feature importance,
its learning performance can be further boosted.

To demonstrate the superiority of the proposed method,
we perform experiments on two benchmark data sets
(MNIST and NEWS20).
Empirical results confirm that
feature-based stacking performs better than sample-based stacking.
It is also better than directly using PLR on the training data set.
Besides,
the prediction performance is further boosted when feature importance is used.
Finally,
we apply the proposed approach
for cross-organizational diabetes prediction
in the transfer learning setting.
The experiment
is performed on the RUIJIN data set,
which contains over ten thousands diabetes records from across China.
Results show
significantly improved
diabetes prediction
performance
over the state-of-the-art,
while
still
protecting
data privacy.

\vspace{3px}
\noindent
\textbf{Notation.}
In the sequel, vectors are denoted by lowercase boldface, and $(\cdot)^{\top}$ denotes transpose of a vector/matrix;
$\sigma(a) = \nicefrac{\exp(a)}{(1 + \exp(a))}$ is the sigmoid function.
A function $g$ is
$\mu$-strongly convex if $g(\alpha\vw + (1-\alpha)\vu) \le \alpha g(\vw) + (1-\alpha) g(\vu) - \frac{\mu}{2}\alpha(1-\alpha)\|\vw-\vu\|^2$
for any $\alpha\in (0, 1)$.

\vspace{-5px}
\section{Related Works}

\subsection{Differential Privacy}
Differential privacy \cite{Dwork2006,dwork2014algorithmic} has been established as a rigorous standard
to guarantee privacy for algorithms that access private data.
Intuitively,
given a privacy budget
$\epsilon$,
an algorithm preserves {\edp} if changing one entry in
the data set does not change the likelihood of any
of the algorithm's
output
by more than $\epsilon$.
Formally, it is defined as follows.

\begin{definition}[\cite{Dwork2006}]
	\label{def:privacy}
	A randomized mechanism $M$ is $\epsilon$-differentially private if for all output $t$ of $M$
	and for all input data $\mD_1, \mD_2$ differing by one element,
	$\pr(M(\mD_1) = t) \le e^\epsilon \, \pr(M(\mD_2) = t)$.
\end{definition}

To meet the {\edp} guarantee,
careful perturbation or
noise usually needs to be added to the learning algorithm.
A smaller $\epsilon$
provides stricter privacy guarantee
but at the expense of heavier noise,
leading to larger performance deterioration \cite{Chaudhuri2011,bassily2014private}.
A relaxed version of $\epsilon$-differentially private, called
$(\epsilon, \delta)$-differentially privacy
in which $\delta$ measures the loss in privacy,
is proposed
\cite{dwork2014algorithmic}.
However,
we focus on the more stringent Definition~\ref{def:privacy}
in this paper.

\vspace{-5px}
\subsection{Privacy-preserving Logistic Regression (PLR)}
\label{sec:plr}

Logistic regression has been popularly used in machine learning \cite{friedman2001elements}.
Various differential privacy
approaches have
been developed for logistic regression.
Examples include
output perturbation \cite{Dwork2006,Chaudhuri2011},
gradient perturbation \cite{Abadi2016}
and objective
perturbation \cite{Chaudhuri2011,bassily2014private}.
In particular,
objective perturbation, which adds designed and random noise to the learning objective,
has both privacy and learning guarantees as well as good empirical performance.

Privacy-preserving logistic regression (PLR)
\cite{Chaudhuri2011}
is the state-of-the-art
model based on objective perturbation.
Given a data set $\mathcal{D} = \{ \vx_i$, $y_i \}_{i = 1}^n$,
where $\vx_i\in {\mathbb R}^d$ is the sample and $y_i$ the corresponding class label,
we first consider the regularized risk minimization problem:
\begin{eqnarray}
\label{eq:reg_problem}
\min_{\vw} \nicefrac{1}{n}\sum\nolimits_{i=1}^n \ell( \vw^{\top}\vx_i, y_i)
+\lambda g(\vw),
\end{eqnarray}
where
$\vw$ is a vector of the model parameter,
$\ell(\hat{y}, y) = \log(1 + e^{-y \hat{y}})$ is the logistic loss
(with predicted label $\hat{y}$ and given label $y$),
$g$ is the regularizer and $\lambda \ge 0$ is a hyperparameter.
To guarantee privacy,
\citeauthor{Chaudhuri2011}
(\citeyear{Chaudhuri2011})
added two extra terms to \eqref{eq:reg_problem}, leading to:
\begin{eqnarray}
\min_{\vw}
\nicefrac{1}{n}\sum\nolimits_{i=1}^n \ell(\vw^{\top}\vx_i, y_i)
\! + \! \nicefrac{\vb^{\top} \! \vw}{n} 
\! + \! \nicefrac{\Delta\|\vw\|^2}{2}
+ \lambda g(\vw),
\label{eq:privatelr}
\end{eqnarray}
where
$\vb$ is random noise
drawn from
$h(\vb) \propto \exp( \nicefrac{\epsilon'}{2} \|\vb\|)$
with $\mathbb{E}( \| \vb \| ) = \nicefrac{2 d}{\epsilon'}$,
$\epsilon'$ is a privacy budget modified from $\epsilon$,
and
$\Delta$ is a scalar depending
on $\lambda$, $n, \epsilon$.
The whole PLR procedure is shown in Algorithm~\ref{alg:plr}.

\begin{algorithm}[ht]
\caption{PLR: Privacy-preserving logistic regression. }
\small
	\begin{algorithmic}[1]
		\REQUIRE privacy budget $\epsilon$, data set $\mD$;

		\STATE $\epsilon' = \epsilon - \log(1 + \nicefrac{1}{2n\lambda} + \nicefrac{1}{16n^2\lambda^2})$;
		\IF{$\epsilon' > 0$}
		\STATE $\Delta = 0$;
		\ELSE
		\STATE $\Delta = (4n(\exp(\nicefrac{\epsilon}{4})-1))^{-1} - \lambda$ and $\epsilon' = \nicefrac{\epsilon}{2}$;
		\ENDIF
		\STATE scale $\|\vx\| \le 1$ for all $\vx \in \mD$;
		\STATE pick a random vector $\vb$ from $h(\vb) \propto \exp{(\nicefrac{\epsilon' \|\vb\|}{2})}$;
		\STATE obtain $\vw$ by solving \eqref{eq:privatelr};

		\RETURN $\vw$.
	\end{algorithmic}
	\label{alg:plr}
\end{algorithm}

\begin{proposition}[\cite{Chaudhuri2011}]
\label{thm:dfgua}
If the regularizer $g$ is strongly convex,
Algorithm~\ref{alg:plr} provides $\epsilon$-differential privacy.
\end{proposition}

While
privacy guarantee
is desirable,
the resultant
privacy-preserving
machine learning model
may not have good learning performance.
In practice,
the performance typically
degrades
dramatically
because of the introduction of noise
\cite{Chaudhuri2011,rajkumar2012differentially,bassily2014private,shokri2015privacy}.
Assume that samples from $\mD$ are drawn i.i.d. from an underlying distribution $P$.
Let $L(\vw; P) = \mathbb{E}_{(\vx, y)\sim P} [\ell(\vw^{\top}\vx, y)]$ be the expected loss of the model.
The following Proposition shows the number of samples needed for PLR
to have
comparable performance as a given baseline model.
This bound is tight
for any $\epsilon$-differential privacy algorithm,
and cannot be further improved \cite{kifer2012private,bassily2014private}.
However,
in contrast, the standard logistic regression model in \eqref{eq:reg_problem} only needs
$n >  \nicefrac{C_1 \|\mathbf{v}\|^2 \log(\frac{1}{\delta})}{\epsilon_g^2}$ samples
\cite{shalev2008svm},
and  thus may be smaller.

\begin{proposition}[\cite{Chaudhuri2011}]
\label{pr:plrl}
Let $g(\cdot) = \nicefrac{1}{2} \| \cdot \|^2$, and $\mathbf{v}$ be a reference model parameter.
Given $\delta > 0$ and $\epsilon_g > 0$, there exists a constant $C_1$ such that when
\begin{equation}
\label{eq:thm_plr}
n \! > \! C_1
\max
\left( \nicefrac{\|\mathbf{v}\|^2\log(\frac{1}{\delta})}{\epsilon_g^2},
\nicefrac{d\log(\frac{d}{\delta})\|\mathbf{v}\|}{\epsilon_g \epsilon},
\nicefrac{\|\mathbf{v}\|^2}{\epsilon_g\epsilon}
\right),
\!\!
\end{equation}
$\vw$ from Algorithm~\ref{alg:plr} meets
$\pr [ L( \vw, \! P ) \! \le \! L( \mathbf{v}, \! P) \! + \! \epsilon_g ] \! \ge \! 1 \! - \! \delta$.
\end{proposition}

\vspace{-5px}
\subsection{Multi-Party Data Learning}
\label{sec:multipart}

Ensemble learning has been considered
with differential privacy under multi-party data learning (MPL).
The task is to
combine predictors from multiple parties
with privacy
\cite{pathak2010multiparty}.
\citeauthor{pathak2010multiparty} [\citeyear{pathak2010multiparty}]
first proposed a specially designed protocol to privately
combine multiple predictions.
The performance is later surpassed by \cite{hamm2016learning,Papernot2017},
which uses another classifier built on auxiliary unlabeled data.
However,
all these combination methods
rely on extra, privacy-insensitive public data,
which may not be always available.
Moreover, the aggregated prediction may not be better than
the best single
party's
prediction.

There are also
MPL methods
that do not use ensemble learning.
\citeauthor{rajkumar2012differentially}
[\citeyear{rajkumar2012differentially}]
used
stochastic gradient descent,
and
\citeauthor{Xie2017} [\citeyear{Xie2017}]
proposed
a multi-task learning method.
While these improve the performance of the previous ones based on aggregation,
they gradually lose the privacy guarantee after more and more iterations.


\vspace{-5px}
\section{Privacy-preserving Ensemble}

In this section,
we propose to improve the learning guarantee of PLR by
ensemble learning \cite{zhou2012ensemble}.
Popular examples
include bagging
\cite{breiman1996bagging}, boosting
\cite{friedman2000additive}, and stacking
\cite{wolpert1992stacked}.
Bagging
and boosting
are based on partitioning of training samples,
and use pre-defined rules
(majority or weighted voting) to combine predictions from models
trained on different partitions.
Bagging improves learning performance by reducing the variance.
However, logistic regression is a convex model and relatively stable.
Boosting, on the other hand, is useful in combining weak models to a strong one,
while logistic regression is a relatively strong classifier and often has good performance in many applications.

In this paper,
we focus on stacking
and show that its privacy-preserving version can be realized based on sample
partitioning
(Section~\ref{sec:basic})
and feature partitioning (FP)
(Section~\ref{sec:ppsg}).
However,
sample partitioning (SP)
may suffer from insufficient training samples, while
FP does not.
Besides,
FP
allows
the incorporation of feature importance to improve learning performance.

\subsection{Privacy-preserving Stacking with Sample Partitioning (SP)}
\label{sec:basic}

We first consider using stacking with SP, and
PLR is used as both the low-level and high-level models
(Algorithm~\ref{alg:simple}).
As stacking does not impose restriction on the usage of classifiers
on each partition of the training data,
a simple combination of stacking and PLR can be used to provide privacy guarantee.


\begin{algorithm}[ht]
\caption{PST-S: Privacy-preserving stacking with SP.}
\small
\begin{algorithmic}[1]
	\REQUIRE privacy budget $\epsilon$, data set $\mD$;

	\STATE partition $\mathcal{D}$ into disjoint sets
	$\mathcal{D}^l$ and $\mathcal{D}^h$, for
	training of the low-level and
	high-level models, respectively;

	\STATE partition samples in $\mathcal{D}^l$ to
	$K$ disjoint sets
	$\{\mathcal{S}_1$, $\dots$, $\mathcal{S}_K\}$;

	\FOR{$k = 1, \dots, K$}

	\STATE train PLR (Algorithm~\ref{alg:plr}) with privacy budget $\epsilon$ on $\mathcal{S}_k$,
	and obtain
	the low-level model parameter
	$\vw_k^l$;
	\ENDFOR

	\STATE construct meta-data set $\mathcal{M}^s = \{ [\sigma(\mathbf{x}^{\top} \vw_1^l)$;$\dots$;$\sigma(\mathbf{x}^{\top}\vw_K^l)], y \}$ using all samples $\{ \vx, y \} \in \mD^h$;

	\STATE train PLR (Algorithm~\ref{alg:plr}) with privacy budget $\epsilon$ on
	$\mathcal{M}^s$,
	and obtain
	the high-level model parameter
	$\vw^h$;

	\RETURN $\{ \vw_k^l \}$ and $\vw^h$.
\end{algorithmic}
\label{alg:simple}
\end{algorithm}

\begin{proposition}
\label{pr:simple}
If the regularizer $g$ is strongly convex,
Algorithm~\ref{alg:simple} provides $\epsilon$-differential privacy.
\end{proposition}

However,
while the high-level
model can be better than any of the single low-level models \cite{dvzeroski2004combining},
Algorithm~\ref{alg:simple} may not
perform
better than directly using PLR on the whole $\mathcal{D}$
for the following two reasons.
First,
each low-level model uses only $\mathcal{S}_k$ (step~4), which is about $\nicefrac{1}{K}$ the size of $\mD$
(assuming that the data set $\mD$ is partitioned uniformly).
This smaller sample size
may not satisfy condition (\ref{eq:thm_plr})
in Proposition~\ref{pr:plrl}.
Second,
in many real-world applications,
features are not of equal importance.
For example, for
diabetes prediction using
the
RUIJIN data set (Table~\ref{tab:ruijin-feature}),
{\sf Glu120} and {\sf Glu0},
which directly measure glucose levels in the blood,
are more relevant than features such as age and number of children.
However,
during training of the low-level models,
Algorithm~\ref{alg:simple} adds equal amounts of noise to all features.
If we can add less noise to the more important features
while keeping the same privacy guarantee,
we are likely to get better learning performance.

\subsection{Privacy-preserving Stacking with Feature Partitioning (FP)}
\label{sec:ppsg}

To address the above problems,
we propose to  partition the data based on
features instead of samples
in training the low-level models.
The proposed feature-based stacking approach is shown in Algorithm~\ref{alg:pstfp}.
Features are partitioned into
$K$ subsets, and
$\mD^l$ is split correspondingly into $K$ disjoint sets
$\{ \mathcal{F}_1, \dots, \mathcal{F}_K \}$.
Obviously, as the number of training samples is not reduced, the sample size
condition for learning performance guarantee is easier to be satisfied
(details will be established in Theorem~\ref{thm:prsplit_plr}).



\begin{algorithm}[ht]
\caption{PST-F:
	Privacy-preserving stacking with FP.}
\small
\begin{algorithmic}[1]
\REQUIRE privacy budget $\epsilon$,
data set $\mD$,
feature importance $\{q_k\}_{k=1}^K$ where $q_k \ge 0$ and $\sum_{k=1}^K{q_k} = 1$;
\footnote{---$q_k$ to partitions}

\STATE partition $\mathcal{D}$ into disjoint sets
$\mD^l$ and $\mD^h$, for training of the
low-level model
and
high-level model, respectively;

\STATE partition $\mD^l$ to
$K$ disjoint sets
$\{ \mathcal{F}_1$, $\dots$, $\mathcal{F}_K\}$
based on features;

\STATE $\epsilon' = \epsilon - \sum_{k=1}^K\log(1 + \nicefrac{q_k^2}{2n \lambda_k} + \nicefrac{q_k^4}{16n^2 \lambda_k^2})$;

\FOR{$k = 1, \dots, K$}
\STATE scale $\| \vx \| \le q_k$ for all $\vx \in \mathcal{F}_k$;
\IF{$\epsilon' >0$}
\STATE $\Delta_k = 0$ and $\epsilon_k = \epsilon'$;
\ELSE
\STATE $\Delta_k = \nicefrac{q_k^2}{4n(\exp(\nicefrac{\epsilon q_k}{4})-1)} - \lambda_k$ and $\epsilon_k = \nicefrac{\epsilon}{2}$;
\ENDIF

\STATE pick a random $\vb_k$ from $h(\vb) \propto \exp( \nicefrac{\epsilon_k\| \vb \|}{2}  ) $;

\STATE
$\vw_k^l  =
\arg\min_{\vw} \!
\nicefrac{1}{n}\sum_{\vx_i \in \mathcal{F}_k} \ell(\vw^{\top}\vx_i, y_i)
+ \nicefrac{\vb_k^{\top} \vw}{n} 
+ \nicefrac{\Delta\|\vw\|^2}{2}
+ \lambda_k g_k(\vw)$;

\ENDFOR

\STATE construct meta-data set $\mathcal{M}^f =$
$\{[\sigma(\vx_{(1)}^{\top} \vw_1^l)$,$\dots$,$\sigma(\vx_{(K)}^{\top} \vw_K^l) ]$, $y\}$
using all $\{ \vx, y \} \in \mD^h$,
where $\vx_{(k)}$ is a vector made from $\vx$  by taking features covered by $\mathcal{F}_k$;

\STATE train PLR (Algorithm~\ref{alg:plr}) with privacy budget $\epsilon$ on $\mathcal{M}^f$,
and obtain the high-level model parameter $\vw^h$;

\RETURN $\{ \vw_k^l \}$ and $\vw^h$.
\end{algorithmic}
\label{alg:pstfp}
\end{algorithm}

When the relative
importance of feature subsets is known,
Algorithm~\ref{alg:pstfp}
adds less noise to the more important features.
Specifically,
let the importance\footnote{When feature importance is not known, $q_1 = \dots = q_K = \nicefrac{1}{K}$.}
of $\mathcal{F}_k$
(with $d_k$ features)
be
$q_k$,
where $q_k \ge 0$ and $\sum_{k = 1} q_k = 1$,
and is independent with $\mathcal{D}$.
Assume that $\epsilon' > 0$ in step~6 (and thus
$\epsilon_k=\epsilon'$).
Recall from Section~\ref{sec:plr} that
$\mathbb{E}(\|\vb_k\|)
= \nicefrac{2 d_k}{\epsilon_k}
= \nicefrac{2 d_k}{\epsilon'}$.
By scaling the samples in each
$\mathcal{F}_k$ as in step~5,
the injected noise level in $\mathcal{F}_k$ is given by
$\nicefrac{\mathbb{E}(\| \vb_k \|)}{\| \vx \|} = \nicefrac{2 d_k}{\epsilon' q_k}$.
This is thus inversely proportional to
the importance $q_k$.

\begin{remark}
In the special case
where
only one feature group
has nonzero importance,
Algorithm~\ref{alg:pstfp} reduces
Algorithm~\ref{alg:plr} on that group,
and privacy is still guaranteed.
\end{remark}

Finally,
a privacy-preserving low-level logistic regression model
is obtained in step~12, and
a privacy-preserving high-level logistic regression model is obtained in step~15.
In general,
each low-level model can have its own
regularizer $g_k$
(an example is shown in Theorem~\ref{thm:plr_fs_learn}).

\noindent\textbf{1). Privacy Guarantee.}
Theorem~\ref{thm:prsplit_plr}
guarantees
privacy of Algorithm~\ref{alg:pstfp}.
Note that the proofs in \cite{Chaudhuri2011,bassily2014private} cannot be
directly used, as they consider neither stacking nor feature importance.

\begin{theorem}
\label{thm:prsplit_plr}
If all $g_k$'s are strongly convex,
Algorithm~\ref{alg:pstfp} provides $\epsilon$-differential privacy.
\end{theorem}

\noindent\textbf{2). Learning Performance Guarantee.}
Analogous to Proposition~\ref{thm:dfgua},
the following bounds the learning performance of each low-level model.

\begin{theorem}
\label{thm:plr_fs_learn}
$g_k = \nicefrac{1}{2}\|\cdot - \mathbf{u}_k\|^2$, where $\mathbf{u}_k$ is any constant
vector,
and $\mathbf{v}_k$ is a
reference model parameter.
Let $a_k = q_k \|\mathbf{v}_k\|$.
given $\delta > 0$ and $\epsilon_g > 0$,
there exists a constant
$C_1$ such that when
\begin{eqnarray} \label{eq:split_bound}
n \! > \!  C_1 \max
\left(
\nicefrac{a_k^2 \log(\nicefrac{1}{\delta})}{\epsilon_g^2},
\nicefrac{ d\log(\nicefrac{d}{K\delta})  a_k}{ q_k K \epsilon_g \epsilon},
\nicefrac{a_k^2}{\epsilon_g\epsilon}
\right),
\end{eqnarray}
$\vw_k^l$ from Algorithm~\ref{alg:pstfp} satisfies
$\pr [ L( \vw_k^l, P ) \le L( \mathbf{v}_k, P ) + \epsilon_g ] \ge 1 - \delta$.
\end{theorem}

\begin{remark}
When $K \! = \! 1$
(a single low-level model trained with all features) and
$\mathbf{u}_k\!=\!\mathbf{0}$,
Theorem~\ref{thm:plr_fs_learn}
reduces to Proposition~\ref{pr:plrl}.
\end{remark}

Note that, to keep the same bound $L(\vv_k, P) + \epsilon_g$, since $\vx$s' are scaled by $q_k$,
$\vv_k$ should be scaled by  $\nicefrac{1}{q_k}$, so
 $\mathbb{E}(a_k) = \mathbb{E}(q_k \|\mathbf{v}_k\|) $ remains the same as  $q_k$ changes.
Thus,
Theorem~\ref{thm:plr_fs_learn} shows that low-level models on more important features can indeed learn better,
if these features are assigned with larger $q_k$.
Since
stacking can have better performance than any single model \cite{ting1999issues,dvzeroski2004combining}
and Theorem~\ref{thm:plr_fs_learn} can offer better learning guarantee than
Proposition~\ref{pr:plrl},
Algorithm~\ref{alg:pstfp} can have better performance than Algorithm~\ref{alg:plr}.
Finally,
compared with Proposition~\ref{thm:dfgua},
$g_k$ in
theorem~\ref{thm:plr_fs_learn} is more flexible in allowing
an extra $\mathbf{u}_k$.
We will show in Section~\ref{sec:xfer} that this is useful for transfer learning.

Since the learning performance
of stacking itself is still an open issue \cite{ting1999issues}, we leave the
guarantee for the whole
Algorithm~\ref{alg:pstfp}
as future work.
A potential problem with FP is that
possible correlations among feature subsets
can no longer be utilized.
However,
as the high-level model can combine information from various low-level models,
empirical results
in Section~\ref{sec:bm-data set} show  that
this is not problematic unless $K$ is very large.

\subsection{Application to Transfer Learning}
\label{sec:xfer}

Transfer learning \cite{Pan2010} is a powerful and promising method to extract
useful knowledge from a source domain to a target domain.
A popular transfer learning approach
is
hypothesis transfer learning (HTL)
\cite{Kuzborskij2013},
which encourages
the hypothesis learned
in the target domain to be similar with that in the source domain.
For application to \eqref{eq:reg_problem},
HTL adds an extra regularizer as:
\begin{align}
\min_{ \vw }
\sum\nolimits_{\vx_i \in \mD_{\text{tgt}}}
\ell(\vw^{\top}\vx_i, y_i)
\!+\! \lambda g(\vw)
\!+\! \nicefrac{\eta}{2} \| \vw \!-\! \vw_{\text{src}} \|^2.
\end{align}
Here, $\eta$ is a hyperparameter,
$\mD_{\text{tgt}}$ is the
target domain
data, and
$\vw_{\text{src}}$ is obtained from the source domain.
Algorithm~\ref{alg:ptrans}
shows how
PST-F can be extended with HTL
using
privacy budgets
$\epsilon_{\text{src}}$ and $\epsilon_{\text{tgt}}$
for the source and target domains, respectively.
The same
feature
partitioning
is used
on both the source and target data.
PLR is trained on each source domain data subset to obtain
$(\vw_{\text{src}})_k$ (steps~2-4).
This is then transferred to the target domain
using PST-F with $g_k(\vw) \! = \! \frac{1}{2}\| \vw \! - \! (\vw_{\text{src}})_k \|^2$ (step~5).

\begin{algorithm}[ht]
	\caption{PST-H: Privacy-preserving stacking with HTL.}
	\small
	\begin{algorithmic}[1]
		\REQUIRE source data sets $\mD_{\text{src}}$, target data set $\mD_{\text{tgt}}$,  and corresponding
		privacy budgets $\epsilon_{\text{src}}$ and $\epsilon_{\text{tgt}}$, respectively.
		\\
		{\bf (source domain processing)}
		\STATE partition $\mD_{\text{src}}$
		to $K$ disjoint sets $\left\lbrace \mathcal{F}_1, \dots, \mathcal{F}_K \right\rbrace $
		based on features;

		\FOR{$k = 1, \dots, K$}
		\STATE
		train PLR with privacy budget $\epsilon_{\text{src}}$ on $\mathcal{F}_k$
		and obtain $(\vw_{\text{src}})_k$;
		\ENDFOR

		{\bf (target domain processing)}

		\STATE obtain $\left\lbrace (\vw_{\text{tgt}})_k^l \right\rbrace $ and $\vw_{\text{tgt}}^h$ from PST-F (Algorithm~\ref{alg:pstfp}) by
		taking $g_k(\vw) \! = \! \nicefrac{1}{2}\| \vw \! - \! (\vw_{\text{src}})_k \|^2$
		and privacy budget $\epsilon_{\text{tgt}}$ on $\mD_{\text{tgt}}$;

		\RETURN $\left\lbrace (\vw_{\text{src}})_k \right\rbrace$ for source domain,
		$\left\lbrace (\vw_{\text{tgt}})_k^l \right\rbrace $ and $\vw_{\text{tgt}}^h$ for target domain.
	\end{algorithmic}
	\label{alg:ptrans}
\end{algorithm}

The following provides privacy guarantees on both the source and target domains.

\begin{corollary}
\label{cor:tansfer}
Algorithm~\ref{alg:ptrans} provides
$\epsilon_{\text{src}}$- and $\epsilon_{\text{tgt}}$-differential privacy guarantees
for the source and target domains.
\end{corollary}

Recently,
privacy-preserving HTL is also proposed in \cite{wangdifferentially}.
However,
it does not consider stacking
and ignores feature importance.

\vspace{-5px}
\section{Experiments}

Experiments are performed on a server with Intel Xeon E5 CPU and 250G memory.
All codes are in Python.

\vspace{-5px}
\subsection{Benchmark Datasets}
\label{sec:bm-data set}

Experiments are performed on
two
popular benchmark
data sets
for evaluating privacy-preserving learning algorithms \cite{shokri2015privacy,Papernot2017,wangdifferentially}:
MNIST
\cite{lecun1998gradient}
and
NEWS20 \cite{lang1995newsweeder}
(Table~\ref{tab:summary_bm}).
The MNIST data set
contains
images of handwritten digits.
Here, we use the digits $0$ and $8$.
We randomly select 5000 samples.
$60\%$ of them are used for training
(with $\nicefrac{1}{3}$ of this used for validation),
and the remaining $20\%$ for testing.
The NEWS20  data set
is a collection of newsgroup documents.
Documents
belonging to the topic ``sci" are taken as positive samples, while those in the topic ``talk" are taken as negative.
Finally,
we use 
\footnote{	Here we use PCA for ablation study.
	However,
	note that the importance scores should be obtained from side information independent from the data or from experts' opinions (as in diabetes example).
	Otherwise, $\epsilon$-differential privacy will not be guaranteed. }
PCA to reduce the feature dimensionality to $100$,
as original dimensionality for MINIST/NEWS20 is too high for differentially private algorithms to handle as the noise will be extremely large.

\begin{table}[ht]
	\small
	\centering
	\vspace{-5px}
	\begin{tabular}{ccc|ccc}
		\hline
		 \multicolumn{3}{c|}{MNIST}   &  \multicolumn{3}{c}{NEWS20}   \\ \hline
		\#train & \#test & \#features & \#train & \#test & \#features \\
		 3000   &  2000   &    100     &  4321 &  643 &    100     \\ \hline
	\end{tabular}
	\vspace{-5px}
	\caption{Summary of the MNIST and NEWS20 data sets. }
	\vspace{-10px}
	\label{tab:summary_bm}
\end{table}

The following algorithms are compared:
(i) PLR, which applies Algorithm~\ref{alg:plr} on the training data;
(ii) PST-S: Algorithm~\ref{alg:simple}, based on SP; and
(iii) PST-F: Algorithm~\ref{alg:pstfp}, based on FP.
We use $K=5$ and
$50\%$ of the data for $\mD^l$ and the remaining
for
$\mD^h$.
Two
PST-F
variants are compared:
PST-F(U),
with random FP and equal feature importance.
And
PST-F(W),
with partitioning based on the PCA feature scores;
and the importance of the $k$th group $\mathcal{F}_k$ is
\begin{eqnarray}
q_k = \nicefrac{\sum\nolimits_{i: f_i \in \mathcal{F}_k} v_i}{\sum\nolimits_{j:f_j \in \mathcal{D}^l} v_j},
\label{eq:weights}
\end{eqnarray}
where $v_i$ is the variance of the $i$th feature $f_i$.
Gradient perturbation is worse than 
objective perturbation in logistic regression \cite{bassily2014private},
thus is not compared.

The area-under-the-ROC-curve (AUC) \cite{Hanley1983}
on the testing set
is used
for performance evaluation.
Hyper-parameters are tuned using the validation set.
To reduce statistical variations,
the experiment
is repeated
$10$ times, and the results averaged.

\noindent
\textbf{1). Varying Privacy Budget $\epsilon$.}
Figure~\ref{fig:bmeps} shows
the testing AUC's
when the privacy budget $\epsilon$ is varied.
As can be seen, the AUCs  for all methods
improve when the privacy requirement is relaxed ($\epsilon$ is large and less noise is added).
Moreover,
PST-S can be inferior to PLR,
due to insufficient training samples caused by SP.
Both PST-F(W) and PST-F(U) have better AUCs than PST-S and PLR.
In particular,
PST-F(W) is the best
as it can utilize feature importance.
Since PST-S
is inferior to PST-F(U),
we only consider PST-F(U) in the following experiments.

\begin{figure}[ht]
	\centering
	\subfigure[MNIST.]
	{\includegraphics[width=\figwidth\columnwidth]{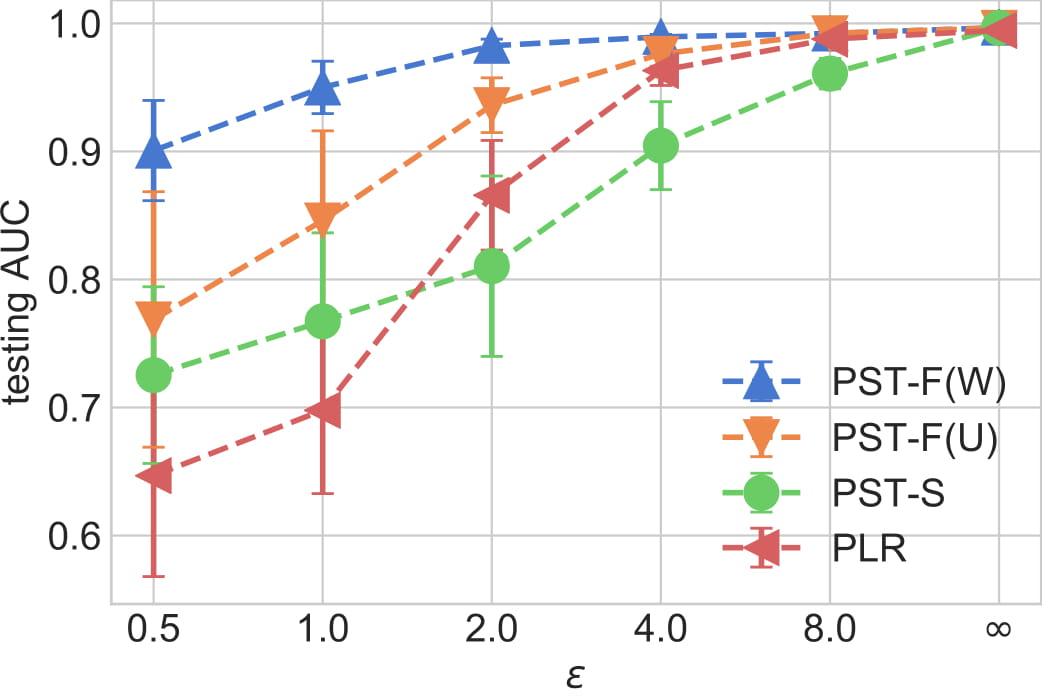}}
	\subfigure[NEWS20.]
	{\includegraphics[width=\figwidth\columnwidth]{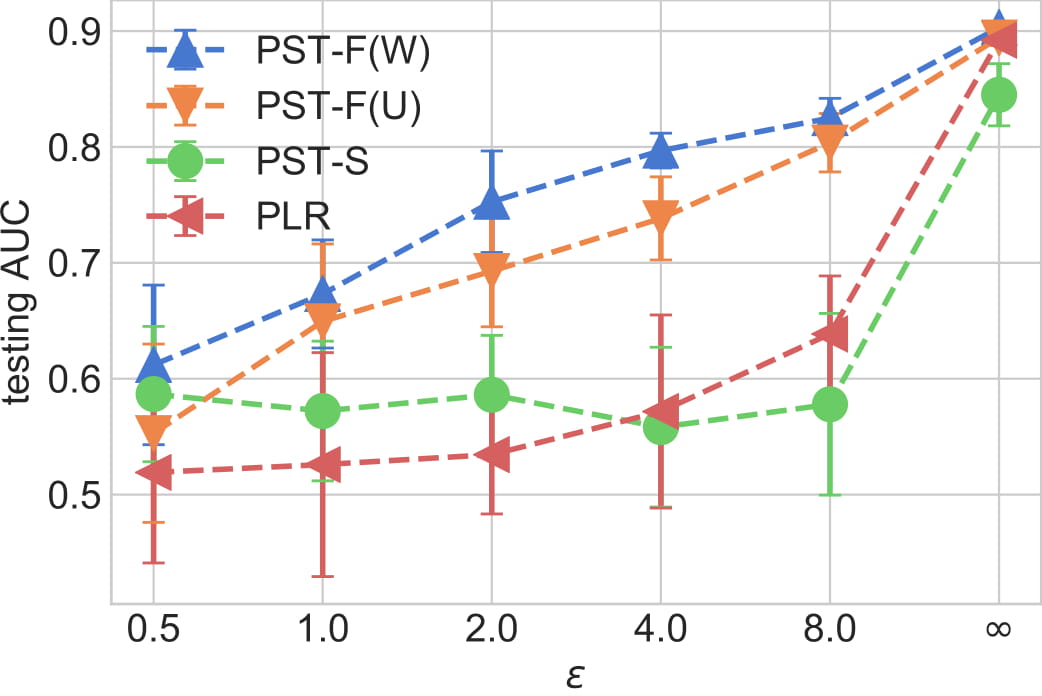}}
	\vspace{-10px}
	\caption{Testing AUC vs $\epsilon$.
Here, ``$\infty$" corresponds to the non-privacy-preserving version of the
		corresponding algorithms.}
	\label{fig:bmeps}
\end{figure}

\noindent
\textbf{2). Varying Number of Partitions $K$.}
In this  experiment,
we fix $\epsilon = 1$,
and vary $K$.
As can be seen
from Figure~\ref{fig:s-k},
when $K$ is very small, ensemble learning is not effective.
When $K$ is too large,
a lot of feature correlation information
is lost and the testing AUC also decreases.


\begin{figure}[ht]
	\centering
	\vspace{-5px}
	\subfigure[MNIST.]
	{\includegraphics[width=\figwidth\columnwidth]{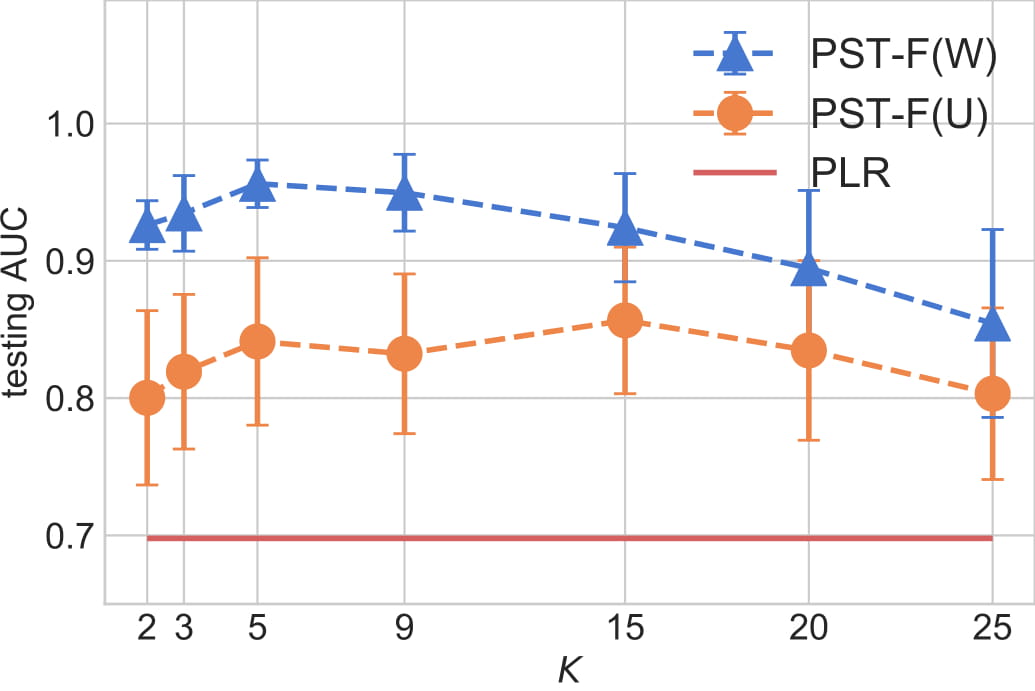}}
	\subfigure[NEWS20.]
	{\includegraphics[width=\figwidth\columnwidth]{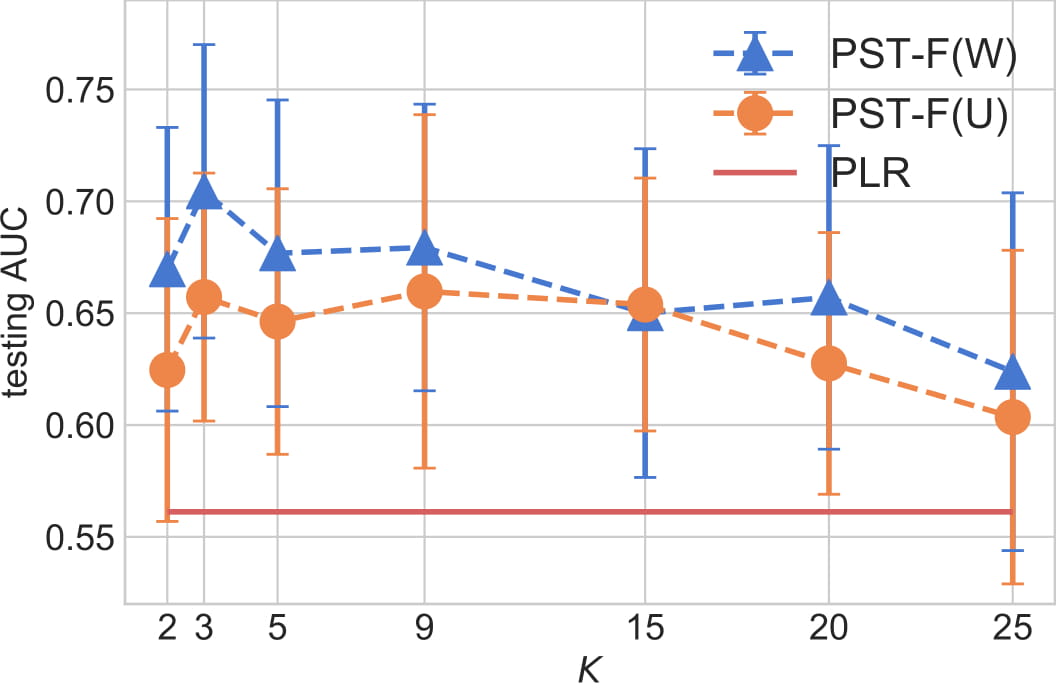}}
	\vspace{-10px}
	\caption{Testing AUC at different $K$'s.}
	\label{fig:s-k}
	\vspace{-10px}
\end{figure}

\noindent
\textbf{3). Changing the Feature Importance.}
In the above experiments,
feature importance is defined based on the variance from PCA.
Here,
we show how feature importance influences prediction performance.
In real-world applications,
we may not know the exact importance of features.
Thus, we
replace variance $v_i$ by
the $i$th power of $\alpha$
($\alpha^i$), where $\alpha$ is a positive constant,
and use \eqref{eq:weights} for assigning weights.
Note that
when $\alpha \! < \! 1$,
more importance features have larger weights;
and
vice versa
when $\alpha \! > \! 1$.
Note that PST-F(W) does not reduce to PST-F(U) when $\alpha = 1$,
as more important features are still grouped together.

\begin{table*}[ht]
\centering
\scalebox{0.8}
{
\begin{tabular}{l|cccccccc}
	\hline
	branch\#                          &            1             &            2             &            3             &            4             &             5             &            6             &             7             &            8              \\ \hline
	PST-H(W)                          & \textbf{0.747$\pm$0.032} & \textbf{0.736$\pm$0.032} & \textbf{0.740$\pm$0.040} & \textbf{0.714$\pm$0.040} & \textbf{0.766$\pm$0.039}  & \textbf{0.707$\pm$0.017} & \textbf{0.721$\pm$0.0464} & \textbf{0.753$\pm$0.042}  \\
	PST-H(U)                          &     0.678$\pm$0.049      & \textbf{0.724$\pm$0.037} &     0.652$\pm$0.103      & \textbf{0.708$\pm$0.033} &      0.653$\pm$0.070      &     0.663$\pm$0.036      & \textbf{0.682$\pm$0.0336} &     0.692$\pm$0.044       \\
	PPHTL                             &     0.602$\pm$0.085      &    {0.608$\pm$0.078}     &     0.528$\pm$0.062      &    {0.563$\pm$0.067}     &      0.577$\pm$0.075      &    {0.601$\pm$0.031}     &    {0.580$\pm$0.0708}     &     0.583$\pm$0.056       \\
	PLR(target)                       &     0.548$\pm$0.088      &    {0.620$\pm$0.055}     &    {0.636$\pm$0.046}     &    {0.579$\pm$0.075}     &      0.533$\pm$0.058      &    {0.613$\pm$0.035}     &    {0.561$\pm$0.0764}     &     0.584$\pm$0.045       \\ \hline
	branch\#                          &            9             &            10            &            11            &            12            &            13             &            14            &            15             &            16             \\ \hline
	PST-H(W)                          & \textbf{0.701$\pm$0.023} & \textbf{0.698$\pm$0.036} & \textbf{0.736$\pm$0.046} & \textbf{0.738$\pm$0.045} & \textbf{0.746$\pm$0.0520} & \textbf{0.661$\pm$0.094} & \textbf{0.697$\pm$0.023}  & \textbf{0.604$\pm$0.012}  \\
	PST-H(U)                          &     0.635$\pm$0.026      &     0.644$\pm$0.050      &     0.635$\pm$0.054      &     0.645$\pm$0.061      & \textbf{0.718$\pm$0.0647} & \textbf{0.644$\pm$0.044} &      0.647$\pm$0.061      &     0.567$\pm$0.036       \\
	PPHTL                             &     0.547$\pm$0.066      &     0.517$\pm$0.075      &     0.565$\pm$0.059      &     0.547$\pm$0.089      &     0.592$\pm$0.0806      &    {0.615$\pm$0.071}     &      0.558$\pm$0.065      &     0.524$\pm$0.027       \\
	PLR(target)                       &     0.515$\pm$0.065      &     0.555$\pm$0.061      &     0.553$\pm$0.066      &     0.520$\pm$0.088      &     0.619$\pm$0.0701      &    {0.563$\pm$0.026}     &      0.558$\pm$0.060      &      0.517$\pm$0.053      \\ \hline
\end{tabular}
}
\vspace{-5px}
\caption{Testing AUC on all branches of RUIJIN data set.
	The best and comparable results according to pair-wise 95\% significance test are high-lighted.
	Testing AUC of PLR on main center is 0.668$\pm$0.026.}
\label{tab:RUIJIN_performance}
\vspace{-10px}
\end{table*}

%



Figure~\ref{fig:s-r} shows the testing AUCs at different $\alpha$'s.
As can be seen, with proper assigned weights
(i.e., $\alpha < 1$ and more important features have larger $q_k$'s),
the testing AUC can get higher.
If less important features are more valued,
the testing AUC decreases
and may not be better than PST-F(U),
which uses uniform weights.
Moreover,
we see that PST-F(W) is not sensitive to the weights
once they are properly assigned.

\begin{figure}[ht]
	\centering
	\vspace{-5px}
	\subfigure[MNIST.]
	{\includegraphics[width=\figwidth\columnwidth]{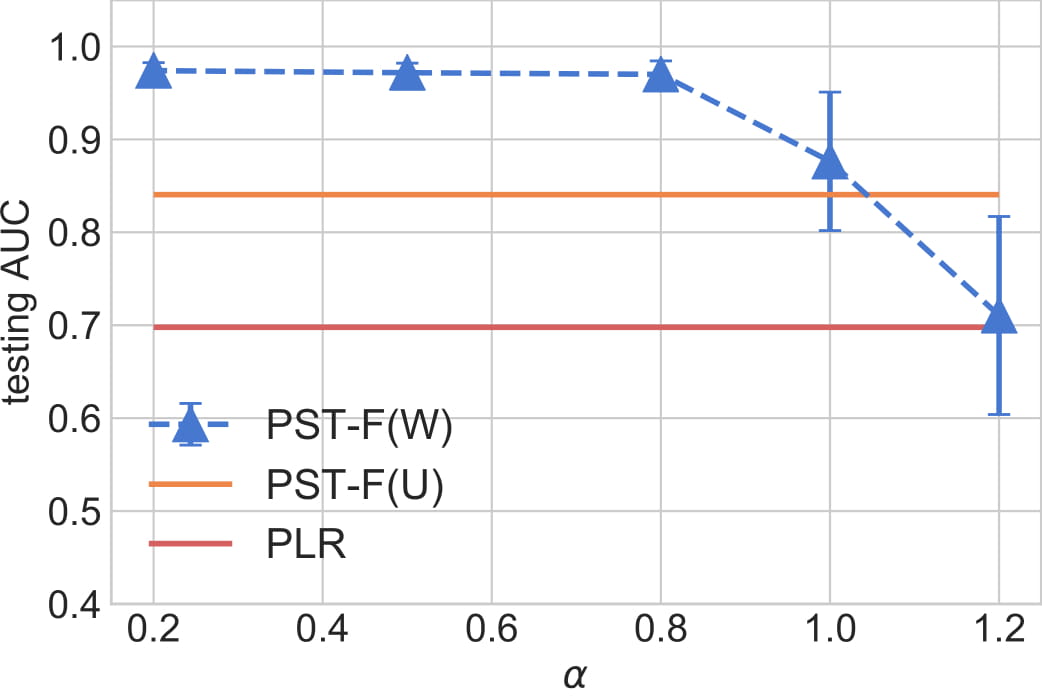}}
	\subfigure[NEWS20.]
	{\includegraphics[width=\figwidth\columnwidth]{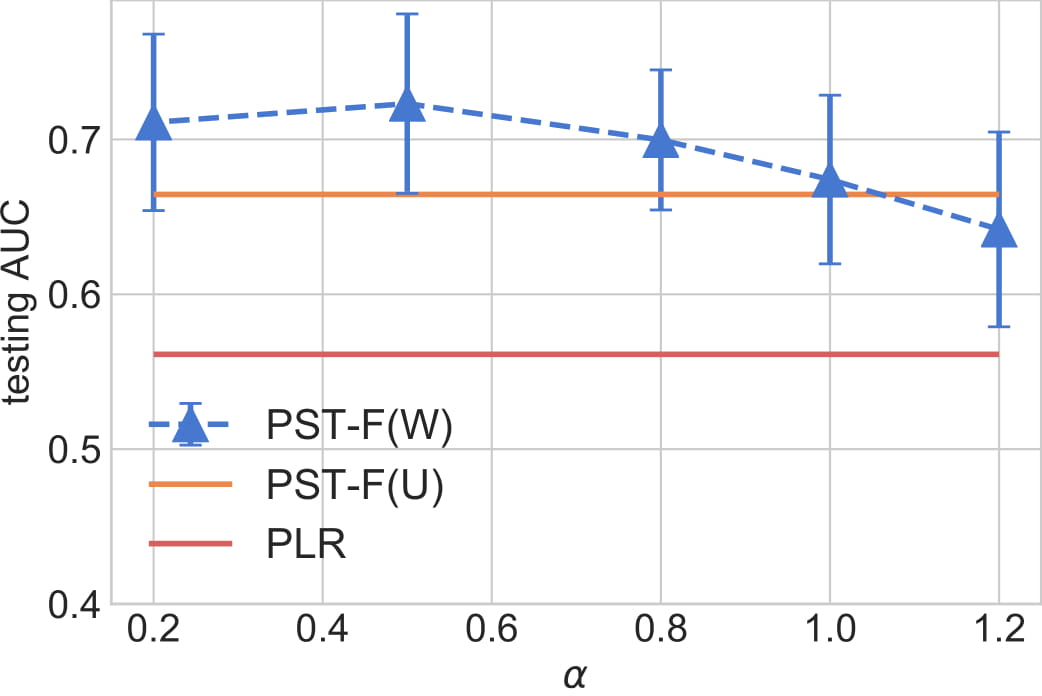}}
	\vspace{-10px}
	\caption{Testing AUC with different feature importance settings.}
	\label{fig:s-r}
	\vspace{-5px}
\end{figure}

\noindent
\textbf{4). Choice of High-Level Model.}
In this section, we compare
different high-level models in
combining predictions from the low-level models.
The following methods are compared:
(i) major voting (C-mv)
from low-level models;
(ii) weighted major voting (C-wmv), which uses $\{q_k\}$ as the weights;
and (iii) by a high-level model in PST-F (denoted ``C-hl'').
Figure~\ref{fig:hlevel}
shows results
on NEWS20 with $\epsilon \! = \! 1.0$.
As can be seen, C-0
in Figure~\ref{fig:hlevel}(b) has the best performance among all single low-level models,
as it contains the most important features.
Besides,
stacking (i.e., C-hl), is the best way to combine predictions from C-\{0-4\},
which also offers better performance than any single low-level models.

\begin{figure}[ht]
	\centering
	\subfigure[PST-F(U).]
	{\includegraphics[width=\figwidth\columnwidth]{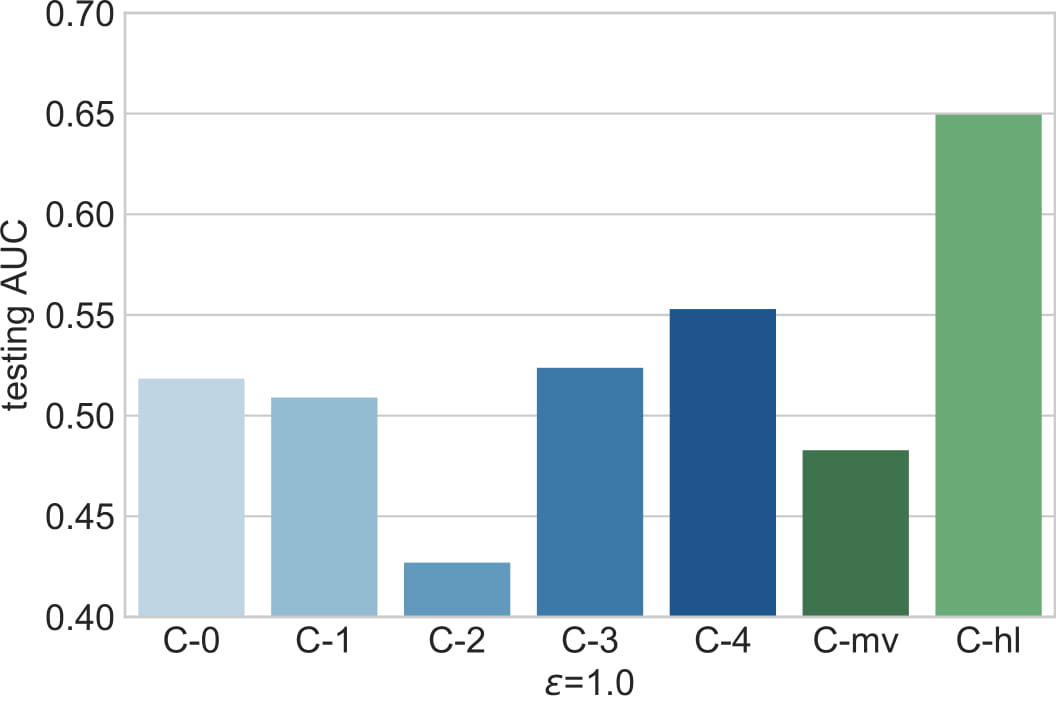}}
	\subfigure[PST-F(W).]
	{\includegraphics[width=\figwidth\columnwidth]{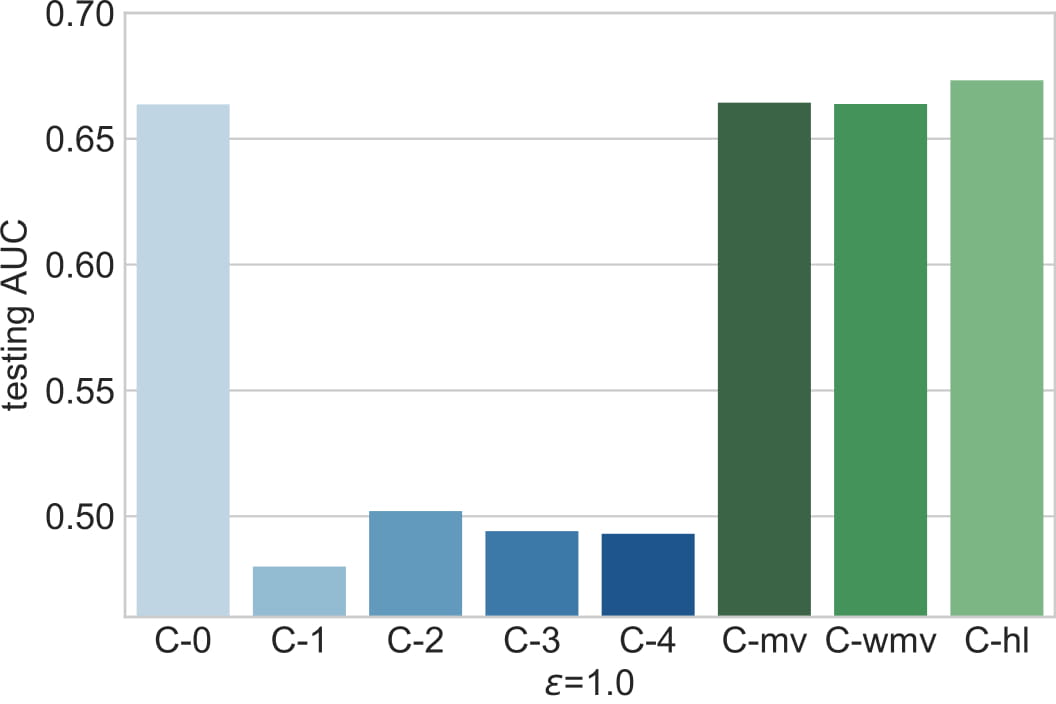}}
	\vspace{-10px}
	\caption{Testing AUC of low-levels models
			and different combining methods on
			NEWS20 ($\epsilon = 1.0$),
			where C-0 to C-4 are performance
			of low-level models.}
	\label{fig:hlevel}
	\vspace{-5px}
\end{figure}

\vspace{-5px}
\subsection{Diabetes Prediction}
\label{ssec:diab}

\noindent
\textbf{1). Background.}
Diabetes is a group of metabolic disorders with high blood sugar levels over a prolonged period.
From 2012 to 2015, approximately 1.5 to 5 million deaths each year are resulted from
diabetes.
Thus,
prevention and diagnosis of diabetes are of great importance.
The RUIJIN
diabetes data set
is collected
by the Shanghai Ruijin Hospital
during two investigations
(in 2010 and 2013), conducted
by the main hospital in Shanghai and 16 branches across China.
The first investigation consists of questionnaires and laboratory tests
collecting demographics,
life-styles,
disease information,
and physical examination results.
The second investigation includes diabetes diagnosis.
Some collected features are shown in Table~\ref{tab:ruijin-feature}.
Table~\ref{tab:ruijin} shows a total of 105,763 participants who appear in both two investigations.
The smaller branches may not have sufficient labeled medical records for good prediction.
Hence, it will be useful to borrow knowledge learned by the
main hospital.
However,
users' privacy is a major concern, and
patients' personal medical records
in the main hospital should not be leaked  to the branches.

\begin{table}[ht]
\vspace{-5px}
\centering
\scalebox{0.8}
{
	\begin{tabular}{l c l}
	\hline
	name   & importance & explaination                              \\ \hline
	mchild & 0.010                          & number of children                    \\
	weight & 0.012                          & birth weight                              \\
	bone   & 0.013                          & bone mass measurement                     \\
	eggw   & 0.005                          & frequency of having eggs                  \\ \hline
	Glu120 & 0.055                          & glucose level 2 hours after meals     \\
	Glu0   & 0.060                          & glucose level immediately after meals \\
	age    & 0.018                          & age                                       \\
	bmi    & 0.043                          & body mass index                           \\
	HDL    & 0.045                          & high-density lipoprotein                  \\ \hline
\end{tabular}
}
\vspace{-5px}
\caption{Some features in the RUIJIN data set,
	and importance is suggested by doctors.
	Top (resp. bottom) part: Features collected from the first (resp. second) investigation.}
\vspace{-5px}
\label{tab:ruijin-feature}
\end{table}

\begin{table}[ht]
\centering
\scalebox{0.8}
{\begin{tabular}{ccccccccc}
	\hline
	     main      &      \#1      &      \#2      &      \#3      &      \#4      &      \#5      &      \#6      &      \#7      &      \#8      \\
	$\!$12,702$\!$ & $\!$4,334$\!$ & $\!$4,739$\!$ & $\!$6,121$\!$ & $\!$2,327$\!$ & $\!$5,619$\!$ & $\!$6,360$\!$ & $\!$4,966$\!$ & $\!$5,793$\!$ \\ \hline
	               &      \#9      &     \#10      &     \#11      &     \#12      &     \#13      &     \#14      &     \#15      &     \#16      \\
	               & $\!$6,215$\!$ & $\!$3,659$\!$ & $\!$5,579$\!$ & $\!$2,316$\!$ & $\!$4,285$\!$ & $\!$6,017$\!$ & $\!$6,482$\!$ & $\!$4,493$\!$ \\ \hline
\end{tabular}}
\vspace{-5px}
\caption{Number of samples collected from the main hospital and 16 branches in the RUIJIN data set.}
\vspace{-5px}
\label{tab:ruijin}
\end{table}


\noindent
\textbf{Setup.}
In this section,
we apply the method in Section~\ref{sec:xfer}
for diabetes prediction.
Specifically,
based on the patient data collected during the first investigation in 2010,
we predict whether he/she will have diabetes diagnosed in 2013.
The main hospital serves as the source domain, and the branches are the target domains.
We set $\epsilon_{\text{src}} = \epsilon_{\text{tgt}} = 1.0$.
The following methods are also compared:
(i) PLR(target), which directly uses PLR on the target data;
(ii) PPHTL \cite{wangdifferentially}:
a recently proposed privacy-preserving HTL method based on PLR;
(iii) PST-F(U): There are 50 features, and they are randomly split into five groups, i.e., $K = 5$,
and each group have equal weights;
(iv) PST-F(W): Features are first sorted by importance,
and then grouped as follows:
The top 10 features are placed in the first group, the
next 10 features go to the second group, and so on.
$q^k$ is set based on \eqref{eq:weights}, with $v_i$ being the  importance
values provided by the doctors.
The other settings are the same as in Section~\ref{sec:bm-data set}.

%

\vspace{3px}
\noindent
\textbf{2). Results.}
Results are shown in Table~\ref{tab:RUIJIN_performance}.
PPHTL may not have better performance than PLR(target),
which is perhaps due to noise introduced in features.
However,
PST-F(U) improves over PPHTL by feature splitting,
and consistently outperforms PLR(target).
PST-F(W),
which considers features importance,
is the best.

\vspace{-5px}
\section{Conclusion}
\label{sec:concl}

In this paper,
we propose a new privacy-preserving machine learning method,
which improves privacy-preserving logistic regression by stacking.
This can be done
by either sample-based or feature-based partitioning of the data set.
We provide theoretical  justifications that the feature-based approach is better and
requires
a smaller sample complexity.
Besides,
when
the importance
of features
is available,
this can further boost the feature-based approach both in theory and practice.
Effectiveness
of the proposed method is verified on both standard
benchmark data sets
and a real-world
cross-organizational diabetes prediction
application.
As a future work,
we will extend the proposed algorithm to other classifiers, such as decision tree and deep networks.


\section*{Acknowledgment}
We acknowledge the support of Hong Kong CERG-16209715.
The first author also thanks Bo Han from Riken for helpful suggestions.

{
\bibliographystyle{named}
\bibliography{lib}
}

\appendix
\input{proof}

\end{document}

%% file: proof.tex
\section{Proof}

\noindent
\textbf{Notation.}
Given two datasets $\mD$ and $\mD'$, $|\mD - \mD'| = 1$ denotes that
 $\mD$ and $\mD'$  differ only $1$ sample.
 Let
 $J(\vw\rightarrow \vb|\mD)$ denote the Jacobian matrix of the
mapping from $\vw$ to $\vb$, when the dataset is $\mD$.
Let  $\bar F(\vw;\mD, \vb, \Delta)$ denote $\sum_{\{\vx_i, y_i\} \in \mD} \ell(\vw^\top\vx_i, y_i)+ \frac{1}{n}\vb^\top\vw+
  \frac{1}{2}\Delta\|\vw\|^2 $, where $n$ is the number of samples in $\mD$.
\subsection{Proposition~\ref{pr:simple}}

\begin{proof}
Note that we  apply PLR algorithm with privacy budget $\epsilon$ on
$\mathcal{S}_k$, so
$\vw_k^l$ is $\epsilon$-differentially private for $\mathcal{S}_k$.
We apply PLR algorithm  on meta-data $\mathcal{M}^s$. So we have 
\begin{equation*}
\frac{\pr(\vw^h|\mD^h)}{\pr(\vw^h|\mD'^h)} =
\frac{\pr(\vw^h|\mathcal{M}^s)}{\pr(\vw^h|\mathcal{M}_s)} = \epsilon.
\end{equation*}
Since
$\{\mathcal{S}_1,\mathcal{S}_2,\ldots, \mathcal{S}_K, \mathcal D^h\}$ are disjoint subsets,
according to Theorem~4 in \cite{McSherry2009},
$ \{\{\vw^l_k\}, \vw^h  \}$
is $\epsilon$-differentially private.
\end{proof}

\subsection{Theorem~\ref{thm:prsplit_plr}}

To prove Theorem~\ref{thm:prsplit_plr}, we first prove the
following Lemma~\ref{lem:split_plr1} and \ref{lem:split_privacy}.
Without of generality,
we assume $g$ is of $1$-strongly convex in the sequel.

\begin{lemma}\label{lem:split_plr1}
    For a dataset $\mD$  and  a vector $\vb$,
    assume that $\ell$ is differentiable and continuous with $|\ell'(z)| \le 1$,
    and $|\ell''(z)| \le c$ for all $z$, and $g$ is $1$-strongly convex,
    $\|\vx_i\| \le q$ for all $\vx_i \in \mD$.
    Let $\mD$ and $\mD'$ be the two datasets which differ in the value of
    the $n$-th item such that
    \begin{align}
    \mD &= \{(\vx_1, y_1), \dots, (\vx_{n-1}, y_{n-1}), (\vx_n, y_n)\}, \\
    \mD' &= \{(\vx_1, y_1), \dots, (\vx_{n-1}, y_{n-1}), (\vx'_n, y'_n)\},
    \end{align}
    Moreover, let $\vb$ and $\vb'$ be two vectors such that
\begin{equation*}
\begin{split}
    \bar\vw &= \arg\min_{\vw} \bar F(\vw;\mD, \vb, \Delta) + \lambda g(\vw)  \\
&=
\arg\min_{\vw} \bar F(\vw;\mD', \vb', \Delta) + \lambda g(\vw)
\end{split}
\end{equation*}
For any $\Delta \ge 0$,
we have
\begin{equation}
    \frac{|\dett(J(\bar{\vw} \rightarrow \vb'|\mD'))|}
    {|\dett(J(\bar{\vw} \rightarrow \vb|\mD))|}
    \le \left(1 + \frac{q^2c}{n(\lambda + \Delta)}\right)^2,
    \label{eq:app1}
\end{equation}
and
\begin{equation}
    \|\vb\| - \|\vb'\| \le 2q.
    \label{eq:app2}
\end{equation}
\end{lemma}
\begin{proof}
      We take the gradient of $F$ to $0$ at $\bar{\vw}$,
    and can obtain
    \begin{equation*}
        \vb = -n\lambda\Delta g(\bar{\vw})
        - \sum_{i}^ny_i\ell'(y_i\vx_i^{\top}\bar{\vw})\vx -n\Delta\bar{\vw}
    \end{equation*}
    Moreover, define matrices $A$ and $E$ as follows:
    \begin{align}
        A &= n\lambda\nabla^2g(\bar{\vw}) +
        \sum_{i=1}^n y_i^2\ell''(y_i\vx^{\top}_i\bar{\vw})\vx_i\vx_i^{\top} + n \Delta I_d, \\
        E &=
        - y_i^2\ell''(y_i\vx^{\top}_i\bar{\vw})\vx_i\vx_i^{\top} +(y'_n)^2\ell''(y'_n\vx'^{\top}_i\bar{\vw})\vx_n'\vx_n'^{\top}
    \end{align}
    From the proof of Theorem~9 in \cite{Chaudhuri2011} we know
    \begin{equation*}
        \begin{split}
    &\frac{|\dett(J(\bar{\vw} \rightarrow \vb'|\mD'))|}
    {|\dett(J(\bar{\vw} \rightarrow \vb|\mD))|} \\
    &=
    |1 + \tau_1(A^{-1}E) + \tau_2(A^{-1}E) + \tau_1(A^{-1}E)\tau_2(A^{-1}E)|,
    \end{split}
    \end{equation*}
    where $\tau_1(), \tau_2()$ denote the largest and second largest
    eigenvalues
    of a matrix. Applying the triangle inequality of trace norm,
    \begin{equation*}
        \begin{split}
        |\tau_1(E)| + |\tau_2(E)| \le& |y_n^2\ell''(y_n\vx_n^{\top}\bar{\vw})|\|\vx_n\|^2 \\
        &+
         |-(y'_n)^2\ell''(y_n\vx_n'^{\top}\bar{\vw})|\|\vx_n'\|^2
        \end{split}
    \end{equation*}
    Then upper bounds on $|y_i|$, $\|\vx_i\|$, and $|\ell''(z)|$ yield
    \begin{equation*}
        |\tau_1(E)| + |\tau_2(E)| \le 2cq^2.
    \end{equation*}
    Therefore $|\tau_1(E)|\cdot |\tau_2(E)| \le c^2q^4$, and
    \begin{equation*}
        \begin{split}
    &\frac{|\dett(J(\bar{\vw} \rightarrow \vb'|\mD'))|}
    {|\dett(J(\bar{\vw} \rightarrow \vb|\mD))|} \\
    &\le 1 + \frac{2cq^2}{n(\lambda+\Delta)} + \frac{c^2q^4}{n^2(\lambda+\Delta)^2}
     = \left(1 + \frac{cq^2}{n(\lambda+\Delta)}\right)^2
        \end{split}
    \end{equation*}
    and we obtain \eqref{eq:app1}.
    For $\vb$ and $\vb'$, we have
    \begin{equation*}
        \vb' - \vb = y_n\ell'(y_n\vx^{\top}_n\bar{\vw})\vx_n
        -y'_n\ell'(y_n\vx_n'^{\top}\bar{\vw})\vx_n'
    \end{equation*}
    Due to that of $|\ell'(\cdot)|\le 1$,$|y_i|\le 1$, $\|\vx_i\|\le q$,
    we
    have
    \begin{equation*}
        \|\vb\| - \|\vb\|' \le \|\vb - \vb'\| \le 2q,
    \end{equation*}
 	and we obtain \eqref{eq:app2}.
\end{proof}

\begin{lemma}\label{lem:split_privacy}
    $\{\vw^{l}_k\}$ in Algorithm~\ref{alg:pstfp} is $\epsilon$-differentially
    private with dataset $\mathcal D^l$.
\end{lemma}

\begin{proof}
    For simplicity, in this proof we ignore the superscript $\cdot^l$.
    The proof follows the proof of Theorem~$9$  in \cite{Chaudhuri2011}.
    Let $\mD$ and $\mD'$ be two datasets of size $n$ and  $|\mD - \mD'|=1$.
    So $|\mF_k - \mF'_k| = 1$ for all $k$.
    We have a set of optimization problems
    \begin{equation*}
        \vw_k = \arg\min_{\vw} \bar F(\vw; \mF_k,\vb_k, \Delta_{k})
        + \lambda_kg_k(\vw), \forall k.
    \end{equation*}
    Since ${\vw_{1}, \vw_{2}, \dots, {\vw}_{K}}$ are independent given the
    dataset, we have
    \begin{equation*}
        \begin{split}
        &\frac{\pr(\{\vw_{k}\}^K_{k=1}|\mD)} {\pr(\{\vw_{(k)}\}^K_{k=1}|\mD')}
        = \prod_{k=1}^K \frac{\pr(\vw_{k}|\mF_k)}{\pr(\vw_{k}|\mF_k')} \\
        &= \prod_{k=1}^K \frac{\pr(\vb_{k}|\mF)} {\pr(\vb'_{k}|\mF_k')}
        \frac{\left|\dett(J_k({\vw}^{(k)}\rightarrow \vb_{k}|\mF_k))\right|^{-1}}
        {\left|\dett(J_k({\vw}_{k}\rightarrow \vb'_{k}|\mF_k'))\right|^{-1}}.
    \end{split}
    \end{equation*}
    By \eqref{eq:app1} in Lemma~\ref{lem:split_plr1}  with upper bound of sample norm $q^{(k)}$,
    \begin{equation*}
        \frac{\left|\dett(J_k({\vw}_{k}\rightarrow \vb_{k}|\mF_k))\right|^{-1}}
        {\left|\dett(J_k({\vw}_{(k)}\rightarrow \vb'_{k}|\mF_k'))\right|^{-1}}
        \le \left(1 + \frac{c (q_{k})^2}{n(\lambda_{k} + \Delta_{k})}\right)^2,
    \end{equation*}
    and
    \begin{equation*}
        \frac{\pr(\vb_{k}|\mD)}{\pr(\vb'_{k}|\mD')}  \le
        e^{\epsilon'(\|\vb_{k}\| - \|\vb'_{k}\|)/2} \le e^{q_{k}\epsilon'}.
    \end{equation*}
    Thus
 \begin{equation*}
        \frac{\pr(\{{\vw}_{k}\}^K_{k=1}|\mD)} {\pr(\{{\vw}_{k}\}^K_{k=1}|\mD')}
        = e^{\epsilon'} \prod_{k=1}^K\left(1+\frac{c (q_{k})^2}{n(\lambda_{k} + \Delta_{k})}\right)^{2}.
 \end{equation*}
 We know that for $\ell(\cdot)$, $c = \frac{1}{4}$.
 When $\Delta=0$,
 $\epsilon = \epsilon' + \sum_{k=1}\log(1 + \frac{(q_{k})^2}{2n\lambda_{k}}
 + \frac{(q_{k})^4}{16n^2(\lambda_{k})^2})$,
 we have
\begin{equation*}
        \frac{\pr(\{\vw_{k}\}^K_{k=1}|\mD)} {\pr(\{{\vw}_{k}\}^K_{k=1}|\mD')}
        = e^{\epsilon}.
 \end{equation*}
 When $\Delta>0$, by definition,
 \begin{equation*}
     \prod_{k=1}^K\left(1+\frac{c(q_{k})^2}{n(\lambda_{k} + \Delta_{k})}\right)^{2}
        = \prod_{k=1}^K e^{q_{k}\epsilon/2}
        =  e^{\epsilon/2},
 \end{equation*}
 \begin{equation*}
     \epsilon = \epsilon' + \epsilon/2.
 \end{equation*}
 As a result,
\begin{equation*}
        \frac{\pr(\{{\vw}_{k}\}^K_{k=1}|\mD)} {\pr(\{{\vw}_{k}\}^K_{k=1}|\mD')}
        = e^{\epsilon}.
 \end{equation*}
\end{proof}
Now we are ready to prove Theorem~\ref{thm:prsplit_plr}.
\begin{proof}
From Lemma~\ref{lem:split_privacy}, we know that $\{\vw^k_l\}$ is
$\epsilon$-differentially private for $\mD_l$. Since we apply PLR on $\mM^f$,
we have
\begin{equation*}
    \frac{\pr(\vw^h|\mD^h)}{\pr(\vw^h|\mD'^h)} =
    \frac{\pr(\vw^h|\mM^f)}{\pr(\vw^h|\mM'^f)} = \epsilon.
\end{equation*}
So $\vw^h$ is $\epsilon$-differentially private for $\mD^h$.
Since $\mD^l$ and $\mD^h$ are disjoint subsets,
according to Theorem~4 in \cite{McSherry2009}, 
$\{\{\vw^l_k\}, \vw^h\}$ is $\epsilon$-differentially private for $\mD$.
\end{proof}

\subsection{Theorem~\ref{thm:plr_fs_learn}}
\begin{proof}
For simplicity, we omit the superscript $\cdot^l$.
Suppose the samples in $\mF$ are i.i.d. drawn according to $P_k$.
We define
\begin{align}
    F_k(\vw, \mD) &= \frac{1}{n}\sum_{(\vx_i, y_i)\in
        \mF_k}\ell(\vw^\top\vx_i, y_i) + \lambda_k g_k(\vw),\\
    \tilde F_k(\vw) &= L(\vw, P_k) + \frac{\lambda_k}{2}\|\vw\|^2.
\end{align}
and let
\begin{align}
    \tilde {\vw}_k   &= \arg\min_{\vw} \tilde F_k(\vw)\\
    (\vw_k)^* &= \arg\min_{\vw} F_k(\vw, \mD).
\end{align}
The results in proof of Theorem~18 in \cite{Chaudhuri2011} shows
\begin{align}
\label{eq:split_L}
        L({\vw}_k) = &L({\vv}_k) + (\tilde F_k({\vw}_k) -
\tilde F_k(\tilde{\vw}_k))\\
&+
(\tilde F_k(\tilde{\vw}_k) - \tilde F_k(\vv_k)) 
+
\frac{\lambda_k}{2}\|\vv_k\|^2
- \frac{\lambda_k}{2}\|{\vw}_k\|^2.
\notag
\end{align}
Let $s_k = \|\vv_k\|$
If $n > \frac{q_k(s_k)^2}{\epsilon_g\epsilon}$ and
$\lambda_k > \frac{\epsilon_g}{(s_k)^2}$, then $n\lambda_k > \frac{q_k}{\epsilon}$,
from the definition of $\epsilon'$ in Algorithm~\ref{alg:pstfp},
\begin{equation*}
    \begin{split}
    \epsilon' &= \epsilon - 2\prod_{k=1}^K\log(1 +
    \frac{(q_k)^2}{4n\lambda_k}) \\
    &= \epsilon - 2\prod_{k=1}^K\log(1 + \frac{q_k\epsilon}{4})
    \ge \epsilon - \frac{\epsilon}{2},
\end{split}
\end{equation*}
where the last step is from the inequality $\log(1+x) < x $ for $x \in [0, 1]$.

From the Lemma~$19$ in \cite{Chaudhuri2011}, we have that with probability at
least $1-\delta$,
\begin{equation*}
    F_k({\vw}_k, \mD)  - F_k(({\vw}_k)^*, \mD) \le \frac{4d^2\log^2(d/(K\delta))(q_k)^2}
    {\lambda_kn^2\epsilon^2}.
\end{equation*}
From \cite{Sridharan2009},
\begin{equation*}
    \begin{split}
        &\tilde F_k(\vw_k) - \tilde F_k(\tilde{\vw}_k) \\
        &\le
        2(F_k({\vw}_k, \mD) - F_k(({\vw}_k)^*, \mD)) +
    \mO(\frac{\log(1/\delta)}{K^2\lambda_kn})\\
    &\le \frac{8d^2\log^2(d/(K\delta))}{K^2\lambda_kn^2\epsilon^2}
    +
    \mO(\frac{(q_k)^2\log(1/\delta)}{\lambda_kn}).
\end{split}
\end{equation*}
By definition of $\tilde{\vw}_k$, we have
$\tilde F(\tilde\vw_k) - \tilde F(\vv_k) \le 0$.
If $\lambda_k = \frac{\epsilon_g}{(s_k)^2}$, then the $4$th term in
\eqref{eq:split_L} is at most $\frac{\epsilon_g}{2}$.  
Finally,
the Theorem
follows by solving for $n$  to make the total excess error at most $\epsilon_g$.
\end{proof}